%% file: emnlp-ijcnlp-2019.tex
\documentclass[11pt,a4paper]{article}
\usepackage[hyperref]{emnlp-ijcnlp-2019}
\usepackage{times}
\usepackage{latexsym}

\input{math_commands.tex}

\usepackage[utf8]{inputenc} 
\usepackage[T1]{fontenc}    
\usepackage{hyperref}       
\usepackage{url}            
\usepackage{booktabs}       
\usepackage{amsfonts}       
\usepackage{nicefrac}       
\usepackage{microtype}      

\usepackage[utf8]{inputenc} 
\usepackage[T1]{fontenc}    
\usepackage{url}            
\usepackage{booktabs}       
\usepackage{amsfonts}       
\usepackage{nicefrac}       
\usepackage{microtype}      
\usepackage{graphicx}
\usepackage{subfigure}
\usepackage{amsmath}
\usepackage{multirow}
\usepackage{amsthm}
\usepackage{amssymb}
\usepackage{mathtools}
\usepackage{color}
\usepackage{xcolor}
\usepackage{MnSymbol}
\usepackage{makecell}
\usepackage{arydshln}
\usepackage[shortlabels]{enumitem}
\usepackage{booktabs}
\usepackage{caption}
\usepackage{wrapfig,lipsum}

\renewcommand{\eqref}[1]{Eq. ({\color{blue} \ref{#1}})}

\usepackage{sidecap}
\usepackage[noorphans,vskip=0ex]{quoting}

\newtheorem{proposition}{Proposition}

\newtheorem{definition}{Definition}

\aclfinalcopy 

\title{Transformer Dissection: A Unified Understanding of \\ Transformer's Attention via the Lens of Kernel}

\author{Yao-Hung Hubert Tsai$^1$\, Shaojie Bai$^1$\, Makoto Yamada$^{34}$\\
{\bf Louis-Philippe Morency$^2$\, Ruslan Salakhutdinov$^1$}\\
\{$^1$Machine Learning Department,$^2$Language Technology Institute\}, Carnegie Mellon University\\
$^3$Kyoto University\, $^4$RIKEN AIP\\
\{yaohungt, shaojieb, morency, rsalakhu\}@cs.cmu.edu, myamada@i.kyoto-u.ac.jp\\
\url{https://github.com/yaohungt/TransformerDissection}}
\date{}

\begin{document}
\maketitle
\begin{abstract}
  Transformer is a powerful architecture that achieves superior performance on various sequence learning tasks, including neural machine translation, language understanding, and sequence prediction. At the core of the Transformer is the attention mechanism, which concurrently processes all inputs in the streams. In this paper, we present a new formulation of attention via the lens of the kernel. To be more precise, we realize that the attention can be seen as applying kernel smoother over the inputs with the kernel scores being the similarities between inputs. This new formulation gives us a better way to understand individual components of the Transformer's attention, such as the better way to integrate the positional embedding. Another important advantage of our kernel-based formulation is that it paves the way to a larger space of composing Transformer's attention. As an example, we propose a new variant of Transformer's attention which models the input as a product of symmetric kernels. This approach achieves competitive performance to the current state of the art model with less computation. In our experiments, we empirically study different kernel construction strategies on two widely used tasks: neural machine translation and sequence prediction.
\end{abstract}

\newcolumntype{P}[1]{>{\centering\arraybackslash}p{#1}}
\input{sections/intro.tex}
\input{sections/method.tex}
\input{sections/exp.tex}

\input{sections/rela.tex}

\section{Conclusions}

In this paper, we presented a kernel formulation for the attention mechanism in Transformer, which allows us to define a larger space for designing attention. As an example, we proposed a new variant of attention which reaches competitive performance when compared to previous state-of-the-art models. Via the lens of the kernel, we were able to better understand the role of individual components in Transformer's attention and categorize previous attention variants in a unified formulation. Among these components, we found the construction of the kernel function acts the most important role, and we studied different kernel forms and the ways to integrate positional embedding on neural machine translation and sequence prediction. We hope our empirical study may potentially allow others to design better attention mechanisms given their particular applications.


\section*{Acknowledgments}
We thank Zhilin Yang for helpful discussion on the positional encoding in Transformer's Attention. This work was supported in part by the DARPA grant FA875018C0150, Office of Naval Research grant N000141812861, AFRL CogDeCON, NSF Awards \#1734868 \#1722822, National Institutes of Health, JST PRESTO program JPMJPR165A, and Apple. We would also like to acknowledge NVIDIA’s GPU support. 

\newpage
{
\fontsize{5}{6}\selectfont
\bibliography{emnlp-ijcnlp-2019}
\bibliographystyle{acl_natbib}
}

\end{document}

%% file: math_commands.tex
\usepackage{amsmath,amsfonts,bm}

\def\eqref#1{equation~\ref{#1}}

\def\1{\bm{1}}

\DeclareMathAlphabet{\mathsfit}{\encodingdefault}{\sfdefault}{m}{sl}
\SetMathAlphabet{\mathsfit}{bold}{\encodingdefault}{\sfdefault}{bx}{n}



%% file: sections/intro.tex
\section{Introduction}
Transformer~\cite{vaswani2017attention} is a relative new architecture which outperforms traditional deep learning models such as Recurrent Neural Networks (RNNs)~\cite{sutskever2014sequence} and Temporal Convolutional Networks (TCNs)~\cite{bai2018empirical} for sequence modeling tasks across neural machine translations~\cite{vaswani2017attention}, language understanding~\cite{devlin2018bert},  sequence prediction~\cite{dai2019transformer}, image generation~\cite{child2019generating}, video activity classification~\cite{wang2018non}, music generation~\cite{huang2018improved}, and multimodal sentiment analysis~\cite{tsai2019multimodal}. 
Instead of performing recurrence (e.g., RNN) or convolution (e.g., TCN) over the sequences, Transformer is a feed-forward model that concurrently processes the entire sequence. At the core of the Transformer is its attention mechanism, which is proposed to integrate the dependencies between the inputs. There are up to three types of attention within the full Transformer model as exemplified with neural machine translation application~\cite{vaswani2017attention}: 1) Encoder self-attention considers the source sentence as input, generating a sequence of encoded representations, where each encoded token has a global dependency with other tokens in the input sequence. 2) Decoder self-attention considers the target sentence (e.g., predicted target sequence for translation) as input, generating a sequence of decoded representations\footnote{The generated sequence can be regarded as a translated sequence (i.e., translating from the encoded sequence), where each generated token depends on all tokens in the encoded sequence.}, where each decoded token depends on previous decoded tokens. 3) Decoder-encoder attention considers both encoded and decoded sequences, generating a sequence with the same length as the decoded sequence. It should be noted that some applications has only the decoder self-attention such as sequence prediction~\cite{dai2019transformer}. In all cases, the Transformer's attentions follow the same general mechanism.

At the high level, the attention can be seen as a weighted combination of the input sequence, where the weights are determined by the similarities between elements of the input sequence. We note that this operation is order-agnostic to the permutation in the input sequence (order is encoded with extra positional embedding~\cite{vaswani2017attention,shaw2018self,dai2019transformer}). The above observation inspires us to connect Transformer's attention to kernel learning~\cite{scholkopf2002learning}: they both concurrently and order-agnostically process all inputs by calculating the similarity between the inputs. Therefore, in the paper, we present a new formulation for Transformer's attention via the lens of kernel. To be more precise, the new formulation can be interpreted as a kernel smoother~\cite{wasserman2006all} over the inputs in a sequence, where the kernel measures how similar two different inputs are. The main advantage of connecting attention to kernel is that it opens up a new family of attention mechanisms that can relate to the well-established literature in kernel learning~\cite{scholkopf2002learning}. As a result, we develop a new variant of attention which simply considers a product of symmetric kernels when modeling non-positional and positional embedding. 

Furthermore, our proposed formulation highlights naturally the main components of Transformer's attention, enabling a better understanding of this mechanism: recent variants of Transformers~\cite{shaw2018self,huang2018music,dai2019transformer,child2019generating,lee2018set,wang2018non,tsai2019multimodal} can be expressed through these individual components. Among all the components, we argue that the most important one is the construction of the kernel function. We empirically study multiple kernel forms and the ways to integrate positional embedding in neural machine translation (NMT) using IWSLT'14 German-English (De-En) dataset~\cite{edunov2017classical} and sequence prediction (SP) using WikiText-103 dataset~\cite{merity2016pointer}.

%% file: sections/method.tex
\section{Attention}
\label{sec:method}
This section aims at providing an understanding of attention in Transformer via the lens of kernel. The inspiration for connecting the kernel~\cite{scholkopf2002learning} and attention instantiates from the observation: both operations concurrently processes all inputs and calculate the similarity between the inputs.
We first introduce the background (i.e., the original formulation) of attention and then provide a new reformulation within the class of kernel smoothers~\cite{wasserman2006all}. Next, we show that this new formulation allows us to explore new family of attention while at the same time offering a framework to categorize previous attention variants~\cite{vaswani2017attention,shaw2018self,huang2018music,dai2019transformer,child2019generating,lee2018set,wang2018non,tsai2019multimodal}.
Last, we present a new form of attention, which requires fewer parameters and empirically reaches competitive performance as the state-of-the-art models.

For notation, we use lowercase representing a vector (e.g., $x$), bold lowercase representing a matrix (e.g., $\mathbf{x}$), calligraphy letter denoting a space (e.g., $\mathcal{X}$), and $S$ denoting a set. To relate the notations in sequence to sequence learning~\cite{vaswani2017attention}, $x$ represents a specific element of a sequence, $\mathbf{x} = [x_1, x_2, \cdots, x_T]$ denotes a sequence of features, $S_{\mathbf{x}}=\{x_{1}, x_{2}, \cdots, x_{T}\}$ represents the set with its elements being the features in sequence $\mathbf{x}$, and we refer the space of set $S_{\mathbf{x}}$ as $\mathcal{S}$. 

\subsection{Technical Background}
\label{subsec:back}
Unlike recurrent computation~\cite{sutskever2014sequence} (i.e., RNNs) and temporal convolutional computation~\cite{bai2018empirical} (i.e., TCNs), Transformer's attention is an {\em order-agnostic} operation given the order in the inputs~\cite{vaswani2017attention}. Hence, in the presentation of the paper, we consider the inputs as a set instead of a sequence. When viewing sequence as a set, we lose the temporal (positional) information in inputs which is often crucial for sequence modeling~\cite{sutskever2014sequence}. As a result, Transformer~\cite{vaswani2017attention} introduced positional embedding to indicate the positional relation for the inputs. Formally, a sequence $\mathbf{x} = [x_1, x_2, \cdots , x_T]$ defines each element as $x_i = (f_i, t_i)$ with $f_i \in \mathcal{F}$ being the non-temporal feature at time $i$ and $t_{i} \in \mathcal{T}$ as an temporal feature (or we called it positional embedding). Note that $f_i$ can be the word representation (in neural machine translation~\cite{vaswani2017attention}), a frame in a video (in video activity recognition~\cite{wang2018non}), or a music unit (in music generation~\cite{huang2018music}). $t_i$ can be a mixture of sine and cosine functions~\cite{vaswani2017attention} or parameters that can be learned during back-propagation~\cite{dai2019transformer,ott2019fairseq}. The feature vector are defined over a joint space $\mathcal{X} := (\mathcal{F}\times \mathcal{T}$). The resulting permutation-invariant set is: $S_\mathbf{x}=\{x_1, x_2, \cdots, x_T\} = \{(f_1, t_1), (f_2, t_2), \cdots, (f_T, t_T)\}$.

Followed the definition by~\citet{vaswani2017attention}, we use queries(q)/keys(k)/values(v) to represent the inputs for the attention. To be more precise, $x_{\{q/k/v\}}$ is used for denoting a query/key/value data in the query/key/value sequence $\mathbf{x}_{\{q/k/v\}}$ ($x_{\{q/k/v\}} \in S_{\mathbf{x}_\{q/k/v\}}$) with $S_{\mathbf{x}_\{q/k/v\}}$ being its set representation. We note that the input sequences are the same ($\mathbf{x}_q=\mathbf{x}_k$) for self-attention and are different ($\mathbf{x}_q$ from decoder and $\mathbf{x}_k$ from encoder) for encoder-decoder attention. 

Given the introduced notation, the attention mechanism in original Transformer~\cite{vaswani2017attention} can be presented as:
\begin{equation}
\begin{split}
&\mathrm{Attention}(x_q \,\,;\,\, S_{\mathbf{x}_k}) \\
= \,&\mathrm{softmax}\left(\frac{x_qW_q(\mathbf{x}_kW_k)^\top}{\sqrt{d_k}} \right)\mathbf{x}_kW_v 
\end{split}
\label{eq:attn}
\end{equation}
with $x_q = f_q + t_q$, $\mathbf{x}_k = \mathbf{f}_k + \mathbf{t}_k$, $W_{q/k/v}$ being the weight, and $d_k$ being the feature dimension of $\mathbf{x}_kW_k$. Decoder self-attention further introduces a mask to block the visibility of elements in $S_{\mathbf{x}_k}$ to $x_q$. Particularly, decoder self-attention considers the decoded sequence as inputs ($\mathbf{x}_k = \mathbf{x}_q$), where the decoded token at time $t$ is not allowed to access the future decoded tokens (i.e., tokens decoded at time greater than $t$). On the contrary, encoder self-attention and decoder-encoder attention consider no additional mask to~\eqref{eq:attn}.

Recent work~\cite{shaw2018self,dai2019transformer,huang2018music,child2019generating,lee2018set,parmar2018image,tsai2019multimodal} proposed modifications to the Transformer for the purpose of better modeling inputs positional relation~\cite{shaw2018self,huang2018music,dai2019transformer}, appending additional keys in $S_{\mathbf{x}_k}$~\cite{dai2019transformer}, modifying the mask applied to~\eqref{eq:attn}~\cite{child2019generating}, or applying to distinct feature types~\cite{lee2018set,parmar2018image,tsai2019multimodal}. These works adopt different designs of attention as comparing to the original form (\eqref{eq:attn}). In our paper, we aim at providing an unified view via the lens of kernel.

\subsection{Reformulation via the Lens of Kernel}
We now provide the intuition to reformulate~\eqref{eq:attn} via the lens of kernel. First, the softmax function can be realized as a probability function for $x_q$ observing the keys $\{x_k\}$s in $S_{\mathbf{x}_k}$ ($S_{\mathbf{x}_k}$ is the set representation of sequence $\mathbf{x}_k$). The probability is determined by the dot product between $x_q$ and $x_k$ with additional mappings $W_q/W_k$ and scaling by $d_k$, which we note the dot-product operation is an instance of kernel function. We also introduce a set filtering function $M(x_q, S_{\mathbf{x}_k}): \mathcal{X} \times \mathcal{S} \rightarrow \mathcal{S}$ which returns a set with its elements that operate with (or are connected/visible to) $x_q$. The filtering function $M(\cdot, \cdot)$ plays as the role of the mask in decoder self-attention~\cite{vaswani2017attention}. Putting these altogether, we re-represent~\eqref{eq:attn} into the following definition.
\begin{definition}
Given a non-negative kernel function $k(\cdot, \cdot): \mathcal{X} \times \mathcal{X} \rightarrow \mathbb{R}^{+}$, a set filtering function $M(\cdot, \cdot): \mathcal{X} \times \mathcal{S} \rightarrow \mathcal{S}$, and a value function $v(\cdot): \mathcal{X} \rightarrow \mathcal{Y}$, the Attention function taking the input of a query feature $x_q \in \mathcal{X}$ is defined as 
\begin{equation}
\begin{split} &\mathrm{Attention}\Big(x_q \,\,;\,\, M(x_q, S_{\mathbf{x}_k})\Big)\\
= \,\,& \sum_{{x_k} \in M(x_q, S_{\mathbf{x}_k})} \frac{k(x_q, x_k)}{\sum_{{x_k}' \in M(x_q, S_{\mathbf{x}_k})}k(x_q, {x_k}')} v(x_k).
\end{split}
\label{eq:attn_set}
\end{equation}
\label{def:attn}
\end{definition}

The Definition~\ref{def:attn} is a class of linear smoothers~\cite{wasserman2006all} with kernel smoothing:
\begin{equation*}
\begin{split} & \sum_{{x_k} \in M(x_q, S_{\mathbf{x}_k})} \frac{k(x_q, x_k)}{\sum_{{x_k}' \in M(x_q, S_{\mathbf{x}_k})}k(x_q, {x_k}')} v(x_k) 
\\
& =\,\,\mathbb{E}_{p(x_k |x_q)}\big[v(x_k)\big] ,
\end{split}
\end{equation*}
where $v(x_k)$ outputs the ``values'' and $p(x_k |x_q) = \frac{k(x_q, x_k)}{\sum_{{x_k}' \in M(x_q, S_{\mathbf{x}_k})}k(x_q, {x_k}')}$ is a probability function depends on $k$ and $N$ when $k(\cdot, \cdot)$ is always positive. In the prior work~\cite{vaswani2017attention}, $k(x_q, x_k) = \mathrm{exp}\left(\langle x_qW_q, x_kW_k \rangle / \sqrt{d_k}\right)$ and $v(x_k) = x_kW_v$. Note that the kernel form $k(x_q, x_k)$ in the original Transformer~\cite{vaswani2017attention} is a asymmetric exponential kernel with additional mapping $W_q$ and $W_k$~\cite{wilson2016deep,li2017mmd}\footnote{We note that rigorous definition of kernel function~\cite{scholkopf2002learning} requires the kernel to be semi-positive definite and symmetric. While in the paper, the discussion on kernel allows it to be non-semi-positive definite and asymmetric. In Section~\ref{sec:exp}, we will examine the kernels which are semi-positive and symmetric.}.

The new formulation defines a larger space for composing attention by manipulating its individual components, and at the same time it is able to categorize different variants of attention in prior work~\cite{shaw2018self,huang2018music,dai2019transformer,child2019generating,lee2018set,wang2018non,tsai2019multimodal}. In the following, we study these components by dissecting~\eqref{eq:attn_set} into: 1) kernel feature space $\mathcal{X}$, 2) kernel construction $k(\cdot, \cdot)$, 3) value function $v(\cdot)$, and 4) set filtering function $M(\cdot, \cdot)$.

\subsubsection{Kernel Feature Space $\mathcal{X}$}
In~\eqref{eq:attn_set}, to construct a kernel on $\mathcal{X}$, the first thing is to identify the kernel feature space $\mathcal{X}$. In addition to modeling sequences like word sentences~\cite{vaswani2017attention} or music signals~\cite{huang2018music}, the Transformer can also be applied to images~\cite{parmar2018image}, sets~\cite{lee2018set}, and multimodal sequences~\cite{tsai2019multimodal}. Due to distinct data types, these applications admit various kernel feature space:

\vspace{1mm}
\noindent {\em (i) Sequence Transformer~\cite{vaswani2017attention,dai2019transformer}:} 
$$\mathcal{X} := (\mathcal{F} \times \mathcal{T})$$ 
    \\with $\mathcal{F}$ being non-positional feature space and $\mathcal{T}$ being the positional embedding space of the position in the sequence.

\vspace{1mm}    
\noindent {\em (ii) Image Transformer~\cite{parmar2018image}:} 
$$
        \mathcal{X} := (\mathcal{F} \times \mathcal{H} \times \mathcal{W}) 
    $$ 
    \\with $\mathcal{F}$ being non-positional feature space, $\mathcal{H}$ being the positional space of the height in an image, and $\mathcal{W}$ being the positional space of the width in an image.

\vspace{1mm}    
\noindent {\em (iii) Set Transformer~\cite{lee2018set} and Non-Local Neural Networks~\cite{wang2018non}:} $$
        \mathcal{X} := (\mathcal{F})
    $$ with no any positional information present.

\vspace{1mm}
\noindent {\em (iv) Multimodal Transformer~\cite{tsai2019multimodal}:} $$
        \mathcal{X} := (\mathcal{F}^\ell\times \mathcal{F}^v\times \mathcal{F}^a\times \mathcal{T})
    $$
    with $\mathcal{F}^\ell$ representing the language feature space, $\mathcal{F}^v$ representing the vision feature space, $\mathcal{F}^a$ representing the audio feature space, and $\mathcal{T}$ representing the temporal indicator space.

For the rest of the paper, we will focus on the setting for sequence Transformer $\mathcal{X} = (\mathcal{F}\times \mathcal{T})$ and discuss the kernel construction on it.

\subsubsection{Kernel Construction and the Role of Positional Embedding $k(\cdot, \cdot)$}
\label{subsubsec:pos}
The kernel construction on $\mathcal{X} = (\mathcal{F}\times \mathcal{T})$ has distinct design in variants of Transformers~\cite{vaswani2017attention,dai2019transformer,huang2018music,shaw2018self,child2019generating}. Since now the kernel feature space considers a joint space, we will first discuss the kernel construction on $\mathcal{F}$ (the non-positional feature space) and then discuss how different variants integrate the positional embedding (with the positional feature space $\mathcal{T}$) into the kernel.

\vspace{2mm}
\hspace{-4mm}{\bf Kernel construction on $\mathcal{F}$.}
All the work considered the scaled asymmetric exponential kernel with the mapping $W_q$ and $W_k$~\cite{wilson2016deep,li2017mmd} for non-positional features $f_q$ and $f_k$:
\begin{equation}
    k_{\mathrm{exp}}(f_q, f_k) = \mathrm{exp}\left(\frac{\langle f_qW_q, f_kW_k \rangle }{\sqrt{d_k}}\right).
\label{eq:fkernel}
\end{equation}
Note that the usage of asymmetric kernel is also commonly used in various machine learning tasks~\cite{yilmaz2007object,tsuda1999support,kulis2011you}, where they observed the kernel form can be flexible and even non-valid (i.e., a kernel that is not symmetric and positive semi-definite). In Section~\ref{sec:exp}, we show that symmetric design of the kernel has similar performance for various sequence learning tasks, and we also examine different kernel choices (i.e., linear, polynomial, and rbf kernel).

\vspace{2mm}
\hspace{-4mm}{\bf Kernel construction on $\mathcal{X} = (\mathcal{F}\times \mathcal{T})$.}
The designs for integrating the positional embedding $t_q$ and $t_k$ are listed in the following.

\vspace{1mm}
\noindent {\em (i) Absolute Positional Embedding~\cite{vaswani2017attention,dai2019transformer,ott2019fairseq}:}
    For the original Transformer~\cite{vaswani2017attention}, each $t_i$ is represented by a vector with each dimension being sine or cosine functions. For learned positional embedding~\cite{dai2019transformer,ott2019fairseq}, each $t_i$ is a learned parameter and is fixed for the same position for different sequences. These works defines the feature space as the direct sum of its temporal and non-temporal space: $\mathcal{X}=\mathcal{F}\bigoplus \mathcal{T}$. Via the lens of kernel, the kernel similarity is defined as 
    \begin{equation}
    k\Big(x_q, x_k\Big) := k_{\mathrm{exp}} \Big(f_q+t_q, f_k+ t_k\Big).
    \label{eq:k1}
    \end{equation}
    
\vspace{1mm}
\noindent {\em (ii) Relative Positional Embedding in Transformer-XL~\cite{dai2019transformer}:} $t$ represents the indicator of the position in the sequence, and the kernel is chosen to be asymmetric of mixing sine and cosine functions:
    \begin{equation}
    k\Big(x_q, x_k\Big) := k_{\mathrm{exp}}\Big(f_q, f_k\Big)\cdot k_{f_q}\Big(t_q, t_k\Big)
    \label{eq:k2}
    \end{equation}
    with $k_{f_q}\Big(t_q, t_k\Big)$ being an asymmetric kernel with coefficients inferred by $f_q$: $\mathrm{log}\,k_{f_q}\Big(t_q, t_k\Big) =$
    $\sum_{p = 0}^{\left \lfloor d_k/2 \right \rfloor -1} c_{2p}\,\sin(\frac{t_q- t_k}{10000^{\frac{2 p}{512}}}) + c_{2p+1}\,\cos(\frac{t_q- t_k}{10000^{\frac{2 p}{512}}})$
    with $[c_0, \cdots, c_{d_k-1}] = f_q W_qW_R$    where $W_R$ is an learned weight matrix. We refer readers to~\citet{dai2019transformer} for more details.

\vspace{1mm}
\noindent {\em (iii) Relative Positional Embedding of~\citet{shaw2018self} and Music Transformer~\cite{huang2018music}:}
  $t_\cdot$ represents the indicator of the position in the sequence, and the kernel is modified to be indexed by a look-up table:
    \begin{equation}
    k\Big(x_q, x_k\Big) :=  L_{t_q-t_k,f_q} \cdot k_{\mathrm{exp}}\Big(f_q, f_k\Big),
    \label{eq:k3}
    \end{equation}
    where $ L_{t_q-t_k,f_q} = \mathrm{exp}(f_q W_q a_{t_q-t_k})$ with $a_\cdot$ being a learnable matrix having matrix width to be the length of the sequence. We refer readers to~\citet{shaw2018self} for more details.

\citet{dai2019transformer} showed that the way to integrate positional embedding is better through~\eqref{eq:k2} than through~\eqref{eq:k3} and is better through~\eqref{eq:k3} than through~\eqref{eq:k1}. We argue the reason is that if viewing $f_i$ and $t_i$ as two distinct spaces \Big($\mathcal{X}:=(\mathcal{F}\times \mathcal{T})$\Big), the direct sum $x_i= f_i + t_i$ may not be optimal when considering the kernel score between $x_q$ and $x_k$. In contrast,~\eqref{eq:k2} represents the kernel as a product of two kernels (one for $f_i$ and another for $t_i$), which is able to capture the similarities for both temporal and non-temporal components. 

\subsubsection{Value Function $v(\cdot)$}
\label{subsubsec:value}
The current Transformers consider two different value function construction:

\vspace{1mm}
\noindent {\em (i) Original Transformer~\cite{vaswani2017attention} and Sparse Transformer~\cite{child2019generating}:}
    \begin{equation}
    v(x_k) = v((f_k, t_k)) := (f_k + t_k)W_v.
    \label{eq:v1}
    \end{equation}
    
\vspace{1mm}
\noindent {\em (ii) Transformer-XL~\cite{dai2019transformer}, Music Transformer~\cite{huang2018music}, Self-Attention with Relative Positional Embedding~\cite{shaw2018self}:}
    \begin{equation}
    v(x_k) = v((f_k, t_k)) := f_kW_v.
    \label{eq:v2}
    \end{equation}
    
Compared~\eqref{eq:v1} to~\eqref{eq:v2},~\eqref{eq:v1} takes the positional embedding into account for constructing the value function. In Section~\ref{sec:exp}, we empirically observe that constructing value function with~\eqref{eq:v2} constantly outperforms the construction with~\eqref{eq:v1}, which suggests that we do not need positional embedding for value function.

\subsubsection{Set Filtering Function $M(\cdot, \cdot)$}
In~\eqref{eq:attn_set}, the returned set by the set filtering function $M(x_q, S_{\mathbf{x}_k})$ defines how many keys and which keys are operating with $x_q$. In the following, we itemize the corresponding designs for the variants in Transformers:

\noindent {\em (i) Encoder Self-Attention in original Transformer~\cite{vaswani2017attention}:} For each query $x_q$ in the encoded sequence, $M(x_q, S_{\mathbf{x}_k}) = S_{\mathbf{x}_k}$ contains the keys being all the tokens in the encoded sequence. Note that encoder self-attention considers $\mathbf{x}_q = \mathbf{x}_k$ with $\mathbf{x}_q$ being the encoded sequence.

\vspace{1mm}
\noindent {\em (ii) Encoder-Decoder Attention in original Transformer~\cite{vaswani2017attention}:} For each query $x_q$ in decoded sequence, $M(x_q, S_{\mathbf{x}_k}) = S_{\mathbf{x}_k}$ contains the keys being all the tokens in the encoded sequence. Note that encode-decoder attention considers $\mathbf{x}_q \neq \mathbf{x}_k$ with $\mathbf{x}_q$ being the decoded sequence and $\mathbf{x}_k$ being the encoded sequence.

\vspace{1mm}
\noindent {\em (iii) Decoder Self-Attention in original Transformer~\cite{vaswani2017attention}:} For each query $x_q$ in the decoded sequence, $M(x_q, S_{\mathbf{x}_k})$ returns a subset of $S_{\mathbf{x}_k}$ ($M(x_q, S_{\mathbf{x}_k}) \subset S_{\mathbf{x}_k}$). Note that decoder self-attention considers $\mathbf{x}_q = \mathbf{x}_k$ with $\mathbf{x}_q$ being the decoded sequence. Since the decoded sequence is the output for previous timestep, the query at position $i$ can only observe the keys being the tokens that are decoded with position $<i$. For convenience, let us define $S_1$ as the set returned by original Transformer~\cite{vaswani2017attention} from $M(x_q, S_{\mathbf{x}_k})$, which we will use it later.

\vspace{1mm}
\noindent {\em (iv) Decoder Self-Attention in Transformer-XL~\cite{dai2019transformer}:} For each query $x_q$ in the decoded sequence, $M(x_q, S_{\mathbf{x}_k})$ returns a set containing $S_1$ and additional memories ($M(x_q, S_{\mathbf{x}_k}) = S_1 + S_{mem}, M(x_q, S_{\mathbf{x}_k}) \supset S_1$). $S_{mem}$ refers to additional memories.

\vspace{1mm}
\noindent {\em (v) Decoder Self-Attention in Sparse Transformer~\cite{child2019generating}:} For each query $x_q$ in the decoded sentence, $M(x_q, S_{\mathbf{x}_k})$ returns a subset of $S_1$ ($M(x_q, S_{\mathbf{x}_k}) \subset S_1$).

To compare the differences for various designs, we see the computation time is inversely proportional to the number of elements in $M(x_q, S_{\mathbf{x}_k})$. For performance-wise comparisons, Transformer-XL~\cite{dai2019transformer} showed that, the additional memories in $M(x_q, S_{\mathbf{x}_k})$ are able to capture longer-term dependency than the original Transformer~\cite{vaswani2017attention} and hence results in better performance. Sparse Transformer~\cite{child2019generating} showed that although having much fewer elements in $M(x_q, S_{\mathbf{x}_k})$, if the elements are carefully chosen, the attention can still reach the same performance as Transformer-XL~\cite{dai2019transformer}. 

\input{tables/tbl1.tex}

\subsection{Exploring the Design of Attention}
\label{subsubsec:explor}
So far, we see how~\eqref{eq:attn_set} connects to the variants of Transformers. By changing the kernel construction in Section~\ref{subsubsec:pos}, we can define a larger space for composing attention. In this paper, we present a new form of attention with a kernel that is 1) valid (i.e., a kernel that is symmetric and positive semi-definite) and 2) delicate in the sense of constructing a kernel on a joint space (i.e., $\mathcal{X} = (\mathcal{F} \times \mathcal{T})$):
\begin{equation}
    \begin{split}
    &k(x_q, x_k) := k_F\Big(f_q, f_k\Big)\cdot k_T\Big(t_q, t_k\Big) \\
    &\mathrm{with}\,\,k_{F}(f_q, f_k) = \mathrm{exp}\Big(\frac{\langle  f_qW_F, f_kW_F \rangle}{\sqrt{d_k}}\Big)  \\
    &\mathrm{and} \,\,\,\, k_{T}(t_q, t_k) = \mathrm{exp}\Big(\frac{\langle  t_qW_T, t_kW_T \rangle}{\sqrt{d_k}}\Big),
    \end{split}
    \label{eq:k4}
\end{equation}
where $W_F$ and $W_T$ are weight matrices. The new form considers product of kernels with the first kernel measuring similarity between non-temporal features and the second kernel measuring similarity between temporal features. Both kernels are symmetric exponential kernel. Note that $t_i$ here is chosen as the mixture of sine and cosine functions as in the prior work~\cite{vaswani2017attention,ott2019fairseq}. In our experiment, we find it reaching competitive performance as comparing to the current state-of-the-art designs (\eqref{eq:k2} by~\citet{dai2019transformer}). We fix the size of the weight matrices $W_\cdot$ in~\eqref{eq:k4} and~\eqref{eq:k2} which means we save $33\%$ of the parameters in attention from ~\eqref{eq:k4} to~\eqref{eq:k2} (\eqref{eq:k2} has weights $W_Q/W_K/W_R$ and~\eqref{eq:k4} has weights $W_F/W_T$).

%% file: tables/tbl1.tex
\begin{table*}[t!]
\centering
\caption{Incorporating Positional Embedding (PE). NMT stands for neural machine translation on IWSLT’14 De-En dataset~\cite{edunov2017classical} and SP stands for sequence prediction on WikiText-103 dataset~\cite{merity2016pointer}. $\uparrow$ means the upper the better and $\downarrow$ means the lower the better.}
\renewcommand{\arraystretch}{1.4}
\resizebox{0.85\textwidth}{!}{
\begin{tabular}{P{50mm}P{27mm}P{40mm}P{25mm}P{26mm}}
\toprule
Approach      & PE Incorporation  & Kernel Form        & NMT (BLEU$\uparrow$)     & SP (Perplexity$\downarrow$) \\ \midrule
\citet{vaswani2017attention} (\eqref{eq:k1}) & Direct-Sum & $k_{\mathrm{exp}} \Big(f_q+t_q, f_k+ t_k\Big)$               & 33.98          & 30.97           \\[2mm]
\citet{shaw2018self} (\eqref{eq:k3})  & Look-up Table & $L_{t_q-t_k,f_q} \cdot k_{\mathrm{exp}}\Big(f_q, f_k\Big)$            & 34.12          & 27.56           \\[2mm]
\citet{dai2019transformer} (\eqref{eq:k2})   & Product Kernel & $k_{\mathrm{exp}}\Big(f_q, f_k\Big)\cdot k_{f_q}\Big(t_q, t_k\Big)$                 & 33.62          & \textbf{24.10}  \\[2mm]
Ours (\eqref{eq:k4})        & Product Kernel & $k_F\Big(f_q, f_k\Big)\cdot k_T\Big(t_q, t_k\Big)$                  & \textbf{34.71} & 24.28 \\[2mm]
\bottomrule 
\end{tabular}}
\label{tbl:1}
\end{table*}

%% file: sections/exp.tex
\section{Experiments}
\label{sec:exp}
By viewing the attention mechanism with~\eqref{eq:attn_set}, we aims at answering the following questions regarding the Transformer's designs:

\vspace{1mm}
\noindent {\bf Q1.} What is the suggested way for incorporating positional embedding in the kernel function?

\vspace{1mm}
\noindent {\bf Q2.} What forms of kernel are recommended to choose in the attention mechanism? Can we replace the asymmetric kernel with the symmetric version?

\vspace{1mm}
\noindent {\bf Q3.} Is there any exception that the attention mechanism is not order-agnostic with respect to inputs? If so, can we downplay the role of positional embedding?

\vspace{1mm}
\noindent {\bf Q4.} Is positional embedding required in value function?

We conduct experiments on neural machine translation (NMT) and sequence prediction (SP) tasks since these two tasks are commonly chosen for studying Transformers~\cite{vaswani2017attention,dai2019transformer}. Note that NMT has three different types of attentions (e.g., encoder self-attention, decoder-encoder attention, decoder self-attention) and SP has only one type of attention (e.g., decoder self-attention). 
For the choice of datasets, we pick IWSLT'14 German-English (De-En) dataset~\cite{edunov2017classical} for NMT and WikiText-103 dataset~\cite{merity2016pointer} for SP as suggested by Edunov {\em et al.}~\cite{edunov2017classical} and Dai {\em et al.}~\cite{dai2019transformer}.
For fairness of comparisons, we train five random initializations and report test accuracy with the highest validation score. We fix the position-wise operations in Transformer\footnote{The computation of Transformer can be categorized into position-wise and inter-positions (i.e., the attention mechanism) operations. Position-wise operations include layer normalization, residual connection, and feed-forward mapping. We refer the readers to Vaswani {\em et al.}~\cite{vaswani2017attention} for more details.} and only change the attention mechanism. Similar to prior work~\cite{vaswani2017attention,dai2019transformer}, we report BLEU score for NMT and perplexity for SP.

\input{tables/tbl2.tex}

\subsection{Incorporating Positional Embedding}
In order to find the best way to integrate positional embedding (PE), we study different PE incorporation in the kernel function $k(\cdot, \cdot)$ in~\eqref{eq:attn_set}. Referring to Sections~\ref{subsubsec:pos} and~\ref{subsubsec:explor}, we consider four cases: 1) PE as direct sum in the feature space (see~\eqref{eq:k1}), 2) PE as a look-up table (see~\eqref{eq:k3}), 3) PE in product kernel with asymmetric kernel (see~\eqref{eq:k2}), and 4) PE in product kernel with symmetric kernel (see~\eqref{eq:k4}). We present the results in Table~\ref{tbl:1}. 

First, we see that by having PE as a look-up table, it outperforms the case with having PE as direct-sum in feature space, especially for SP task. Note that the look-up table is indexed by the relative position (i.e., $t_q - t_k$) instead of absolute position. Second, we see that PE in the product kernel proposed by Dai {\em et al.}~\cite{dai2019transformer} may not constantly outperform the other integration types (it has lower BLEU score for NMT). Our proposed product kernel reaches the best result in NMT and is competitive to the best result in SP. 

\input{tables/tbl3.tex}
\input{tables/tbl4.tex}

\subsection{Kernel Types}

To find the best kernel form in the attention mechanism, in addition to the exponential kernel (see~\eqref{eq:fkernel}), we compare different kernel forms (i.e., linear, polynomial, and rbf kernel) for the non-positional features. We also provide the results for changing asymmetric to the symmetric kernel, when forcing $W_q = W_k$, so that the resulting kernel is a valid kernel~\cite{scholkopf2002learning}. The numbers are shown in Table~\ref{tbl:2}. Note that, for fairness, other than manipulating the kernel choice of the non-positional features, we fix the configuration by Vaswani {\em et al.}~\cite{vaswani2017attention} for NMT and the configuration by Dai {\em et al.}~\cite{dai2019transformer} for SP.

We first observe that the linear kernel does not converge for both NMT and SP. 
We argue the reason is that the linear kernel may have negative value and thus it violates the assumption in kernel smoother that the kernel score must be positive~\cite{wasserman2006all}. Next, we observe the kernel with infinite feature space (i.e., exponential and rbf kernel) outperforms the kernel with finite feature space (i.e., polynomial kernel). And we see rbf kernel performs the best for NMT and exponential kernel performs the best for SP. We conclude that the choice of kernel matters for the design of attention in Transformer. Also, we see no much performance difference when comparing asymmetric to symmetric kernel. In the experiment, we fix the size of $W_\cdot$ in the kernel, and thus adopting the symmetric kernel benefits us from saving parameters.

\subsection{Order-Invariance in Attention}

The need of the positional embedding (PE) in the attention mechanism is based on the argument that the attention mechanism is an order-agnostic (or, permutation equivariant) operation~\cite{vaswani2017attention,shaw2018self,huang2018music,dai2019transformer,child2019generating}. However, we show that, for decoder self-attention, the operation is not order-agnostic. For clarification, we are not attacking the claim made by the prior work~\cite{vaswani2017attention,shaw2018self,huang2018music,dai2019transformer,child2019generating}, but we aim at providing a new look at the order-invariance problem when considering the attention mechanism with masks (masks refer to the set filtering function in our kernel formulation). In other words, previous work did not consider the mask between queries and keys when discussing the order-invariance problem~\cite{perez2019turing}.

To put it formally, we first present the definition by~\citet{lee2018set} for a permutation equivariance function:
\begin{definition}
Denote $\Pi$ as the set of all permutations over $[n]=\{1,\cdots,n\}$. A function $func: \mathcal{X}^n\rightarrow \mathcal{Y}^n$ is permutation equivariant iff for any permutation $\pi \in \Pi$, $func(\pi x) = \pi func(x)$.
\end{definition}
\citet{lee2018set} showed that the standard attention (encoder self-attention~\cite{vaswani2017attention,dai2019transformer} ) is permutation equivariant. Here, we present the non-permutation-equivariant problem on the decoder self-attention:
\begin{proposition}
Decoder self-attention~\cite{vaswani2017attention,dai2019transformer} is not permutation equivariant.
\label{prop:decoder}
\end{proposition}

To proceed the proof, we need the following definition and propositions.

\begin{definition}
Denote $\Pi$ as the set of all permutations over $[n]=\{1,\cdots,n\}$ and $S_{\mathbf{x}_k}^{\pi}$ as performing permutation $\pi$ over $S_{\mathbf{x}_k}$. $\mathrm{Attention}(x_q; S_{\mathbf{x}_k})$ is said to be permutation equivariant w.r.t. $S_{\mathbf{x}_k}$ if and only if for any $\pi \in \Pi$, $\mathrm{Attention}(x_q; S_{\mathbf{x}_k}^{\pi}) = \mathrm{Attention}(x_q; S_{\mathbf{x}_k})$.
\label{def:permu}
\end{definition}

\begin{proposition}
Attention with the set filtering function $M(x_q, S_{\mathbf{x}_k}) = S_{\mathbf{x}_k}$ is permutation equivariant w.r.t. $S_{\mathbf{x}_k}$.
\end{proposition}
\begin{proof}
It is easy to show that if $M(x_q, S_{\mathbf{x}_k}) = S_{\mathbf{x}_k}$,~\eqref{eq:attn_set} remains unchanged for any permutation $\pi$ performed on $S_{\mathbf{x}_k}$.
\end{proof}

\begin{proposition}
Attention with the set difference $S_{\mathbf{x}_k} \setminus M(x_q, S_{\mathbf{x}_k}) \neq \phi$ is not permutation equivariant w.r.t. $S_{\mathbf{x}_k}$.
\label{prop:notequi}
\end{proposition}
\begin{proof}
First, suppose that $\hat{x} \in S_{\mathbf{x}_k} \setminus M(x_q, S_{\mathbf{x}_k})$. Then, we construct a permutation $\pi$ such that $\hat{x} \in M(x_q, S_{\mathbf{x}_k}^\pi)$. It is obvious that~\eqref{eq:attn_set} changes after this permutation and thus $\mathrm{Attention}\Big(x_q \,\,;\,\, M(x_q, S_{\mathbf{x}_k})\Big)$ is not permutation equivariant w.r.t. $S_{\mathbf{x}_k}$.
\end{proof}

\begin{proof}{[Proof for Proposition~\ref{prop:decoder}]}
First, we have $x_q\sim S_{\mathbf{x}_k}$. Hence, showing $\mathrm{Attention}(x_q; S_{\mathbf{x}_k})$ not permutation equivariant w.r.t. $S_{\mathbf{x}_k}$ equals to showing $\mathrm{Attention}$ not permutation equivariant. Then, since the decoder self-attention considers masking (i.e., $M(x_q, S_{\mathbf{x}_k})$ returns a subset of $S_{\mathbf{x}_k}$), by Proposition~\ref{prop:notequi}, the decoder self-attention is not permutation equivariant.
\end{proof}

In fact, not only being a permutation inequivariant process, the decoding process in the decoder self-attention already implies the order information from the data. To show this, take the decoded sequence $\mathbf{y} = [\mathrm{init}, y_1, y_2, y_3, y_4]$ as an example. $\mathrm{init}$ stands for the initial token. When determining the output $y_1$ from $\mathrm{init}$, the set filtering function is $M(\mathrm{init}, S_\mathbf{y}) = \{\mathrm{init}\}$. Similarly, we will have $M(y_1, S_\mathbf{y}), M(y_2, S_\mathbf{y}), M(y_3, S_\mathbf{y})$ to be $\{\mathrm{init}, y_1\},  \{\mathrm{init}, y_1, y_2\},  \{\mathrm{init}, y_1, y_2, y_3\}$. Then, it raises a concern: do we require PE in decoder self-attention? By removing PE in decoder self-attention, we present the results in Table~\ref{tbl:3}. From the table, we can see that, for NMT, removing PE only in decoder self-attention results in slight performance drop (from $34.71$ to $34.49$). However, removing PE in the entire model greatly degrades the performance (from $34.71$ to $14.47$). On the other hand, for SP, removing PE from our proposed attention variant dramatically degrades the performance (from $24.28$ to $30.92$). Nonetheless, the performance is slightly better than considering PE from the original Transformer~\cite{vaswani2017attention}. 

\subsection{Positional Embedding in Value Function}
To determine the need of positional embedding (PE) in value function, we conduct the experiments by adopting~\eqref{eq:v1} or~\eqref{eq:v2} in the attention mechanism. The results are presented in Table~\ref{tbl:4}. From the table, we find that considering PE in value function (\eqref{eq:v1}) does not gain performance as compared to not considering PE in value function (\eqref{eq:v2}). 

\subsection{Take-Home Messages}
Based on the results and discussions, we can now answer the questions given at the beginning of this section. The answers are summarized into the take-home messages in the following.

\vspace{1mm}
\noindent {\bf A1.} We show that integrating the positional embedding in the form of product kernel (\eqref{eq:k2} or~\eqref{eq:k4}) gives us best performance.

\vspace{1mm}
\noindent {\bf A2.} The kernel form does matter. Adopting kernel form with infinite feature dimension (i.e., exponential kernel or rbf kernel) gives us best results. The symmetric design of the kernel may benefit us from saving parameters and barely sacrifice the performance as compared to the non-symmetric one.

\vspace{1mm}
\noindent {\bf A3.} The decoder self-attention is not an order-agnostic operation with respect to the order of inputs. However, incorporating positional embedding into the attention mechanism may still improve performance.

\vspace{1mm}
\noindent {\bf A4.} We find that there is no much performance difference by considering or not considering the positional embedding in value function.

%% file: tables/tbl2.tex
\begin{table*}[t!]
\centering
\caption{Kernel Types. Other than manipulating the kernel choice of the non-positional features, we fix the configuration by~\citet{vaswani2017attention} for NMT and the configuration by~\citet{dai2019transformer} for SP.}
\renewcommand{\arraystretch}{1.4}
\resizebox{0.9\textwidth}{!}{
\begin{tabular}{P{20mm}P{40mm}P{33mm}P{33mm}P{33mm}P{33mm}}
\toprule
\multirow{2}{*}{Type} & \multirow{2}{*}{Kernel Form} & \multicolumn{2}{c}{NMT (BLEU$\uparrow$)} & \multicolumn{2}{c}{SP (Perplexity$\downarrow$)} \\
                      &                              & Asym. ($W_q \neq W_k$)        & Sym. ($W_q = W_k$)      & Asym. ($W_q \neq W_k$)        & Sym. ($W_q = W_k$)           \\ \midrule
Linear                & $\langle f_a W_q, f_b W_k \rangle $                            & not converge    & not converge & not converge    & not converge      \\[3mm]
Polynomial            & $\Big(\langle f_a W_q, f_b W_k \rangle \Big)^2$                            & 32.72           & 32.43        & 25.91           & 26.25             \\[3mm]
Exponential           & $ \mathrm{exp}\Big(\frac{\langle f_a W_q, f_b W_k \rangle}{\sqrt{d_k}}\Big)$                            & 33.98           & 33.78        & 24.10           & \textbf{24.01}    \\[3mm]
RBF                   & $ \mathrm{exp}\Big(-\frac{\|f_a W_q - f_b W_k\|^2}{\sqrt{d_k}}\Big)$                            & \textbf{34.26}  & 34.14        & 24.13           & 24.21 \\[3mm]
\bottomrule 
\end{tabular}}
\label{tbl:2}
\end{table*}

%% file: tables/tbl3.tex
\begin{table*}[t!]
\centering
\caption{Order-Invariance in Attention. To save the space, we denote Encoder Self-Attention / Encoder-Decoder Attention / Decoder Self-Attention as A/B/C. Note that SP only has decoder self-attention.}
\renewcommand{\arraystretch}{1.4}
\resizebox{0.9\textwidth}{!}{
\begin{tabular}{P{40mm}P{35mm}P{25mm}}
\toprule
Approach & Positional Embedding                            & NMT (BLEU$\uparrow$)  \\ \midrule
Ours (\eqref{eq:k4})         & In A/B/C                           & \textbf{34.71}      \\
Ours (\eqref{eq:k4})       & In A/B & 34.49                \\
No Positional Embedding   & none                            & 14.47     
 \\
\bottomrule 
\end{tabular}
\quad
\begin{tabular}{P{45mm}P{35mm}P{26mm}}
\toprule
Approach & Positional Embedding                            & SP (Perplexity$\downarrow$) \\ \midrule
\citet{vaswani2017attention} (\eqref{eq:k1}) & In C                              & 30.97           \\
Ours (\eqref{eq:k4}      & In C                           & \textbf{24.28}           \\
No Positional Embedding  & none   & 30.92          
 \\
\bottomrule 
\end{tabular}
}
\label{tbl:3}
\end{table*}

%% file: tables/tbl4.tex
\begin{table*}[t!]
\centering
\caption{Positional Embedding in Value Function. }
\renewcommand{\arraystretch}{1.4}
\resizebox{0.9\textwidth}{!}{
\begin{tabular}{P{50mm}P{42mm}P{32mm}P{42mm}P{32mm}}
\toprule
\multicolumn{5}{c}{I: Value Function Considering Positional Embedding (\eqref{eq:v1}) / II: Value Function Considering no Positional Embedding (\eqref{eq:v2}) }  \\ \cline{1-5}
\multirow{2}{*}{Approach} & \multicolumn{2}{c}{NMT (BLEU$\uparrow$)} & \multicolumn{2}{c}{SP (Perplexity$\downarrow$)} \\
                          & I \Big($v(x_k) := (f_k + t_k)W_V$\Big)        & II \Big($v(x_k) := f_kW_V$\Big)                 & I \Big($v(x_k) := (f_k + t_k)W_V$\Big)        & II \Big($v(x_k) := f_kW_V$\Big)                   \\ \midrule
\citet{vaswani2017attention} (\eqref{eq:k1})                       & 33.98     & 34.02              & 30.97        & 30.50                \\
\citet{shaw2018self} (\eqref{eq:k3})                     & 34.04     & 34.12              & 27.56        & 27.45                \\
\citet{dai2019transformer} (\eqref{eq:k2})                       & 33.32     & 33.62              & 24.18        & \textbf{24.10}       \\
Ours (\eqref{eq:k4})                     & 34.60     & \textbf{34.71}     & 24.42        & 24.28               
 \\
\bottomrule 
\end{tabular}}
\label{tbl:4}
\end{table*}

%% file: sections/rela.tex
\section{Related Work}


Other than relating Transformer's attention mechanism with kernel methods, the prior work~\cite{wang2018non,shaw2018self,tsai2019video} related the attention mechanism with graph-structured learning. For example, Non-Local Neural Networks~\cite{wang2018non} made a connection between the attention and the non-local operation in image processing~\cite{buades2005non}. Others~\cite{shaw2018self,tsai2019video} linked the attention to the message passing in graphical models. In addition to the fundamental difference between graph-structured learning and kernel learning, the prior work~\cite{wang2018non,shaw2018self,tsai2019video} focused on presenting Transformer for its particular application (e.g., video classification~\cite{wang2018non} and neural machine translation~\cite{shaw2018self}). Alternatively, our work focuses on presenting a new formulation of Transformer's attention mechanism that gains us the possibility for understanding the attention mechanism better.

%% file: emnlp-ijcnlp-2019.bbl
\begin{thebibliography}{26}
\expandafter\ifx\csname natexlab\endcsname\relax\def\natexlab#1{#1}\fi

\bibitem[{Bai et~al.(2018)Bai, Kolter, and Koltun}]{bai2018empirical}
Shaojie Bai, J~Zico Kolter, and Vladlen Koltun. 2018.
\newblock An empirical evaluation of generic convolutional and recurrent
  networks for sequence modeling.
\newblock \emph{arXiv preprint arXiv:1803.01271}.

\bibitem[{Buades et~al.(2005)Buades, Coll, and Morel}]{buades2005non}
Antoni Buades, Bartomeu Coll, and J-M Morel. 2005.
\newblock A non-local algorithm for image denoising.
\newblock In \emph{2005 IEEE Computer Society Conference on Computer Vision and
  Pattern Recognition (CVPR'05)}, volume~2, pages 60--65. IEEE.

\bibitem[{Child et~al.(2019)Child, Gray, Radford, and
  Sutskever}]{child2019generating}
Rewon Child, Scott Gray, Alec Radford, and Ilya Sutskever. 2019.
\newblock Generating long sequences with sparse transformers.
\newblock \emph{arXiv preprint arXiv:1904.10509}.

\bibitem[{Dai et~al.(2019)Dai, Yang, Yang, Cohen, Carbonell, Le, and
  Salakhutdinov}]{dai2019transformer}
Zihang Dai, Zhilin Yang, Yiming Yang, William~W Cohen, Jaime Carbonell, Quoc~V
  Le, and Ruslan Salakhutdinov. 2019.
\newblock Transformer-xl: Attentive language models beyond a fixed-length
  context.
\newblock \emph{arXiv preprint arXiv:1901.02860}.

\bibitem[{Devlin et~al.(2018)Devlin, Chang, Lee, and
  Toutanova}]{devlin2018bert}
Jacob Devlin, Ming-Wei Chang, Kenton Lee, and Kristina Toutanova. 2018.
\newblock Bert: Pre-training of deep bidirectional transformers for language
  understanding.
\newblock \emph{arXiv preprint arXiv:1810.04805}.

\bibitem[{Edunov et~al.(2017)Edunov, Ott, Auli, Grangier, and
  Ranzato}]{edunov2017classical}
Sergey Edunov, Myle Ott, Michael Auli, David Grangier, and Marc'Aurelio
  Ranzato. 2017.
\newblock Classical structured prediction losses for sequence to sequence
  learning.
\newblock \emph{arXiv preprint arXiv:1711.04956}.

\bibitem[{Huang et~al.(2018{\natexlab{a}})Huang, Vaswani, Uszkoreit, Shazeer,
  Hawthorne, Dai, Hoffman, and Eck}]{huang2018improved}
Cheng-Zhi~Anna Huang, Ashish Vaswani, Jakob Uszkoreit, Noam Shazeer, Curtis
  Hawthorne, Andrew~M Dai, Matthew~D Hoffman, and Douglas Eck.
  2018{\natexlab{a}}.
\newblock An improved relative self-attention mechanism for transformer with
  application to music generation.
\newblock \emph{arXiv preprint arXiv:1809.04281}.

\bibitem[{Huang et~al.(2018{\natexlab{b}})Huang, Vaswani, Uszkoreit, Simon,
  Hawthorne, Shazeer, Dai, Hoffman, Dinculescu, and Eck}]{huang2018music}
Cheng-Zhi~Anna Huang, Ashish Vaswani, Jakob Uszkoreit, Ian Simon, Curtis
  Hawthorne, Noam Shazeer, Andrew~M Dai, Matthew~D Hoffman, Monica Dinculescu,
  and Douglas Eck. 2018{\natexlab{b}}.
\newblock Music transformer: Generating music with long-term structure.

\bibitem[{Kulis et~al.(2011)Kulis, Saenko, and Darrell}]{kulis2011you}
Brian Kulis, Kate Saenko, and Trevor Darrell. 2011.
\newblock What you saw is not what you get: Domain adaptation using asymmetric
  kernel transforms.
\newblock In \emph{CVPR 2011}, pages 1785--1792. IEEE.

\bibitem[{Lee et~al.(2018)Lee, Lee, Kim, Kosiorek, Choi, and Teh}]{lee2018set}
Juho Lee, Yoonho Lee, Jungtaek Kim, Adam~R Kosiorek, Seungjin Choi, and
  Yee~Whye Teh. 2018.
\newblock Set transformer.
\newblock \emph{arXiv preprint arXiv:1810.00825}.

\bibitem[{Li et~al.(2017)Li, Chang, Cheng, Yang, and P{\'o}czos}]{li2017mmd}
Chun-Liang Li, Wei-Cheng Chang, Yu~Cheng, Yiming Yang, and Barnab{\'a}s
  P{\'o}czos. 2017.
\newblock Mmd gan: Towards deeper understanding of moment matching network.
\newblock In \emph{Advances in Neural Information Processing Systems}, pages
  2203--2213.

\bibitem[{Merity et~al.(2016)Merity, Xiong, Bradbury, and
  Socher}]{merity2016pointer}
Stephen Merity, Caiming Xiong, James Bradbury, and Richard Socher. 2016.
\newblock Pointer sentinel mixture models.
\newblock \emph{arXiv preprint arXiv:1609.07843}.

\bibitem[{Ott et~al.(2019)Ott, Edunov, Baevski, Fan, Gross, Ng, Grangier, and
  Auli}]{ott2019fairseq}
Myle Ott, Sergey Edunov, Alexei Baevski, Angela Fan, Sam Gross, Nathan Ng,
  David Grangier, and Michael Auli. 2019.
\newblock fairseq: A fast, extensible toolkit for sequence modeling.
\newblock In \emph{Proceedings of NAACL-HLT 2019: Demonstrations}.

\bibitem[{Parmar et~al.(2018)Parmar, Vaswani, Uszkoreit, Kaiser, Shazeer, Ku,
  and Tran}]{parmar2018image}
Niki Parmar, Ashish Vaswani, Jakob Uszkoreit, {\L}ukasz Kaiser, Noam Shazeer,
  Alexander Ku, and Dustin Tran. 2018.
\newblock Image transformer.
\newblock \emph{arXiv preprint arXiv:1802.05751}.

\bibitem[{P{\'e}rez et~al.(2019)P{\'e}rez, Marinkovi{\'c}, and
  Barcel{\'o}}]{perez2019turing}
Jorge P{\'e}rez, Javier Marinkovi{\'c}, and Pablo Barcel{\'o}. 2019.
\newblock On the turing completeness of modern neural network architectures.
\newblock \emph{arXiv preprint arXiv:1901.03429}.

\bibitem[{Scholkopf and Smola(2001)}]{scholkopf2002learning}
Bernhard Scholkopf and Alexander~J Smola. 2001.
\newblock \emph{Learning with kernels: support vector machines, regularization,
  optimization, and beyond}.
\newblock MIT press.

\bibitem[{Shaw et~al.(2018)Shaw, Uszkoreit, and Vaswani}]{shaw2018self}
Peter Shaw, Jakob Uszkoreit, and Ashish Vaswani. 2018.
\newblock Self-attention with relative position representations.
\newblock \emph{arXiv preprint arXiv:1803.02155}.

\bibitem[{Sutskever et~al.(2014)Sutskever, Vinyals, and
  Le}]{sutskever2014sequence}
Ilya Sutskever, Oriol Vinyals, and Quoc~V Le. 2014.
\newblock Sequence to sequence learning with neural networks.
\newblock In \emph{Advances in neural information processing systems}, pages
  3104--3112.

\bibitem[{Tsai et~al.(2019{\natexlab{a}})Tsai, Bai, Liang, Morency, and
  Salakhutdinov}]{tsai2019multimodal}
Yao-Hung~Hubert Tsai, Shaojie Bai, Paul~Pu Liang, Louis-Philippe Morency, and
  Ruslan Salakhutdinov. 2019{\natexlab{a}}.
\newblock Multimodal transformer for unaligned multimodal language sequences.
\newblock \emph{ACL}.

\bibitem[{Tsai et~al.(2019{\natexlab{b}})Tsai, Divvala, Morency, Salakhutdinov,
  and Farhadi}]{tsai2019video}
Yao-Hung~Hubert Tsai, Santosh Divvala, Louis-Philippe Morency, Ruslan
  Salakhutdinov, and Ali Farhadi. 2019{\natexlab{b}}.
\newblock Video relationship reasoning using gated spatio-temporal energy
  graph.
\newblock \emph{CVPR}.

\bibitem[{Tsuda(1999)}]{tsuda1999support}
Koji Tsuda. 1999.
\newblock Support vector classifier with asymmetric kernel functions.
\newblock In \emph{in European Symposium on Artificial Neural Networks (ESANN}.
  Citeseer.

\bibitem[{Vaswani et~al.(2017)Vaswani, Shazeer, Parmar, Uszkoreit, Jones,
  Gomez, Kaiser, and Polosukhin}]{vaswani2017attention}
Ashish Vaswani, Noam Shazeer, Niki Parmar, Jakob Uszkoreit, Llion Jones,
  Aidan~N Gomez, {\L}ukasz Kaiser, and Illia Polosukhin. 2017.
\newblock Attention is all you need.
\newblock In \emph{Advances in Neural Information Processing Systems}, pages
  5998--6008.

\bibitem[{Wang et~al.(2018)Wang, Girshick, Gupta, and He}]{wang2018non}
Xiaolong Wang, Ross Girshick, Abhinav Gupta, and Kaiming He. 2018.
\newblock Non-local neural networks.
\newblock In \emph{Proceedings of the IEEE Conference on Computer Vision and
  Pattern Recognition}, pages 7794--7803.

\bibitem[{Wasserman(2006)}]{wasserman2006all}
Larry Wasserman. 2006.
\newblock \emph{All of nonparametric statistics}.
\newblock Springer Science \& Business Media.

\bibitem[{Wilson et~al.(2016)Wilson, Hu, Salakhutdinov, and
  Xing}]{wilson2016deep}
Andrew~Gordon Wilson, Zhiting Hu, Ruslan Salakhutdinov, and Eric~P Xing. 2016.
\newblock Deep kernel learning.
\newblock In \emph{Artificial Intelligence and Statistics}, pages 370--378.

\bibitem[{Yilmaz(2007)}]{yilmaz2007object}
Alper Yilmaz. 2007.
\newblock Object tracking by asymmetric kernel mean shift with automatic scale
  and orientation selection.
\newblock In \emph{2007 IEEE Conference on Computer Vision and Pattern
  Recognition}, pages 1--6. IEEE.

\end{thebibliography}
